\pdfoutput=1
\documentclass{article}

\usepackage{arxiv}
\usepackage{color}
\usepackage[utf8]{inputenc} % allow utf-8 input
\usepackage[T1]{fontenc}    % use 8-bit T1 fonts
\usepackage{hyperref}       % hyperlinks
\usepackage{url}            % simple URL typesetting
\usepackage{booktabs}       % professional-quality tables
\usepackage{amsfonts}       % blackboard math symbols
\usepackage{nicefrac}       % compact symbols for 1/2, etc.
\usepackage{microtype}      % microtypography
\usepackage{lipsum}
\usepackage{caption}
\usepackage{subcaption}
\usepackage{graphicx}
\usepackage{epsfig}
\usepackage{amsmath}
\usepackage{amssymb}
\usepackage{amsthm}
\graphicspath{ {./images/} }
\usepackage{mathtools}
\usepackage[]{algorithm2e}
\usepackage{algorithmicx}\usepackage{url}
\newcommand{\defeq}{\vcentcolon=}

\newtheorem{thm}{Theorem}
\newtheorem{lemma}{Lemma}
\newtheorem{cor}{Corollary}
\title{The Flag Manifold as a Tool for Analyzing and Comparing Data Sets}

\author{
Xiaofeng Ma\\
  Department of Mathematics\\
  Colorado State University\\
  Fort Collins CO 80523, USA \\
  \texttt{xiaofeng.ma@rams.colostate.edu} \\
   \And
Michael Kirby \\
  Department of Mathematics\\
  Colorado State University\\
  Fort Collins CO 80523, USA \\
  \texttt{michael.kirby@colostate.edu} \\
  \And
 Chris Peterson \\
  Department of Mathematics\\
  Colorado State University\\
  Fort Collins CO 80523, USA \\
  \texttt{christopher2.peterson@colostate.edu} \\
}

\begin{document}

%\author{ Xiaofeng Ma \and Michael Kirby \and Chris Peterson \\
%Colorado State University, Department of Mathematics\\
%Fort Collins CO 80523, USA\\
%{\tt\small xiaofeng.ma@rams.colostate.edu \and michael.kirby@colostate.edu %\and christopher2.peterson@colostate.edu}
%\and

\maketitle
%\thispagestyle{empty}

%%%%%%%%% ABSTRACT
\begin{abstract}
%Motivation:
The shape and orientation of data clouds reflect
variability in observations that can confound pattern recognition systems. 
%Approach:
Subspace methods, utilizing Grassmann manifolds, have been a great aid in dealing with such variability. However, this usefulness begins to falter when the data cloud contains sufficiently many outliers corresponding to stray elements from another class or when the number of data points is larger than the number of features.
We illustrate how nested subspace methods, utilizing flag manifolds, can help to deal with such additional confounding factors. Flag manifolds, which are parameter spaces for nested subspaces, are a natural geometric generalization of Grassmann manifolds. To make practical comparisons on a flag manifold,
algorithms are proposed for determining the distances between points $[A], [B]$ on a flag manifold, where $A$ and $B$ are arbitrary orthogonal matrix representatives for $[A]$ and $[B]$, and for determining the initial direction of these minimal length geodesics.  The approach is illustrated in the
context of (hyper) spectral 
imagery showing the impact of ambient
dimension, sample dimension, and flag structure.
%Conclusions:
%The approach is shown to
%have advantages, as compared to the use of Grassmann manifolds, for analyzing and comparing highly corrupted data sets and for comparing data sets where 
%the number of data points is larger than the
%number of features.
\end{abstract}

%%%%%%%%% BODY TEXT

%---------------------------------------------------------------------------------------------------%
\section{Introduction}
Variability in data observations due, for example, to  image lighting,  data noise, or batch effects, contributes to the challenge of pattern recognition.
One way to approach modeling this variation
is to observe the sample over
its variation in state.   This motivates 
the robust modeling of a set
of data, i.e., modeling specifically to  capture the variability of different realizations
of a data class.  
Practically, one can often exploit this variability by considering a collection of observations abstractly as a single point in an appropriate parameter space and
algorithmically exploiting the geometry of the parameter space.
%Further, one can collect a set of 
%sets of data that can be used to encode
%multiclass variability.

Ideas from geometry and topology have shown considerable promise for the analysis of large, and or complex, data sets given their ability to encode this variability.  For example, the mathematical framework of 
the Grassmannian has proven to be effective 
at capturing many of the pattern variations that so often
confound pattern recognition systems.  In this setting 
data is encoded as subspaces
and distances are measured
using angles between subspaces.
The Grassmann manifold is often a suitable tool for analyzing data sets 
where the number of feature dimensions in the ambient space is less than half of the ambient dimension.
%For instance, objects under variation in illumination can readily be modeled by a representative subspace in low-dimensions allowing the
%classification problem to become a computation 
%of angles between subspaces.

Initially explored in the setting of
subspace packing problems \cite{strohmer2003,conway1996,kutyniok2009robust}, the application of Stiefel and Grassmann manifolds 
has become widespread in
computer vision and pattern recognition.
Examples include:  video processing
\cite{he2012incremental},
classification, 
\cite{harandi2011graph,chang2008classification,wang2009kernel,wang2013grassmannian}, action recognition
 \cite{azary2013}, expression analysis
 \cite{taheri2011towards,turaga2009locally,liu2013partial}, domain
 adaptation
 \cite{kumar2016robust,patel2015visual},
 regression \cite{shaw2013regression,hong2014geodesic},
 pattern recognition \cite{ma2011recognition},
%hyperspectral image analysis \cite{chepushtanova2014classification}, 
and computation of subspace means
\cite{chakraborty2015recursive,marrinan2015flag}.  
More recently, Grassmannians have also been explored in the deep neural
network literature \cite{huang2018building}.
Much of this progress has hinged on the development of efficient
algorithms \cite{edelman1998geometry,gallivan2003efficient,absil} allowing procedures developed in other settings to be transported to  analogous procedures on Grassmann manifolds.
A collection of papers by Nishimori et al introduced flag manifolds
in the context of independent component analysis and optimization \cite{nishimori2006riemannian,nishimori2006riemannian2, nishimori2007flag,nishimori2008natural}. Later work by others used and extended some of these ideas in a variety of contexts \cite{fiori2011extended, draper2014flag, marrinan2015flag, marrinan2016flag, ma2019self}.
Very recent work of Ye, Wong, and Lim gives an expanded view of the local differential geometry of flag manifolds with a very practical viewpoint \cite{2019arXiv190700949Y}. Two features that we were unable to find in the above cited papers, and that were needed in order to develop a particular class of procedures, are
 algorithms for determining the distances between points $[A], [B]$ on a flag manifold where $A$ and $B$ are arbitrary orthogonal matrix representatives for $[A]$ and $[B]$ and algorithms for determining how to move from $[A]$ to $[B]$ along a {\it minimal} length geodesic. In this paper we develop such algorithms and illustrate their use in several sample problems in data analysis.

From the data analysis perspective, points on a Grassmann manifold $Gr(k,n)$ parameterize the $k$-dimensional linear subspaces of $\mathbb R^n$. Points on a flag manifold $FL(n_1,n_2, \dots, n_d)$ parameterize sequences of nested linear subspaces ${0}=V_0\subset V_1\subset V_2\subset \dots \subset V_d=\mathbb R^n$ with $n_i=dim (V_{i})-dim (V_{i-1})$. Flag manifolds can be viewed as generalizations or refinements of Grassmannians and have the ability to encode more subtle relationships  than are capable with Grassmannians.   
In practice, the Grassmannian seems to be well suited
for data sets where the ambient dimension is much
larger than the number of data points (tall matrices) and where the data set is relatively pure. While applicable in this setting, the 
flag manifold approach
is also suitable to the analysis
of some data sets where the data dimension may be small
relative to the number of observations (wide matrices) and where the data set may consist of a mixture of classes.

As described above, flag manifolds constitute a refinement of Grassmann manifolds that enable the  measurement of the distance between nested spaces. They are particularly effective for studying the challenging problem of comparing mixed data sets. An example of what is meant by this is the following: suppose that one data set has 80 percent of its samples drawn from class A and 20 percent from class B and a second data set has the reverse mixture. Grassmann methods have difficulties distinguishing between such data sets whereas flag methods appear to be more robust with respect to distinguishing between these data sets.

%Contributions:
Mathematically, as is demonstrated in this paper, 
the tools for measuring
geodesic distances between data represented by tall versus wide matrices are utilized in a different manner.  
%{\color{red} In our framework, tall data matrices are characterized by having an ambient dimension that is greater than the number of data points, while the opposite is true for wide data matrices.} 
Here we propose
practical algorithms for computing 
distances between wide matrices that may 
be useful for solving pattern recognition and computer vision problems.
The work is in the same spirit as Grassmannian
data processing but extends these tools 
to a distinct yet important application.
We argue that in many cases where data is subject to wide variability, the distances measured
between large sets of small feature spaces captures
more fidelity than algorithms on Euclidean space.

The outline of this paper is as follows:  In Section \ref{GRtext} we
review the geometric framework
of the Grassmannian.
In Section \ref{FLtext} the theory of the
flag manifold is developed along with efficient algorithms to compute geodesic
distances.  In Section \ref{results}
we illustrate the applicability 
of the method on hyperspectral 
imagery.  In Section \ref{conc}
we summarize the features of the 
methodology.

%---------------------------------------------------------------------------------------------------%

\section{The Grassmannian}
\label{GRtext}
The Grassmannian, denoted by $Gr(k,n)$, is a geometric object whose points parameterize the $k$-dimensional subspaces of a fixed $n$-dimensional vector space. In the context of applications, the fixed $n$-dimensional vector space is typically taken  to be $\mathbb{R}^n$ or $\mathbb{C}^n$ (though vector spaces over other fields can also be considered).
For the purposes of this paper, the ambient vector space is taken to be $\mathbb{R}^n$ and we can represent $Gr(k,n)$ as a {\it real matrix manifold}. Each point in $Gr(k,n)$ is identified with an equivalence class of orthogonal matrices leading to the representation of $Gr(k,n)$ as $O(n)/O(k)\times O(n-k)$ or alternatively in terms of special orthogonal matrices as $SO(n)/S(O(k)\times O(n-k))$. In these formulas, $O(n)$ denotes the group of $n\times n$ orthogonal matrices and $O(k)\times O(n-k)$ denotes the subgroup of $O(n)$ consisting of block diagonal matrices with elements from $O(k)$ in the first block and elements from $O(n-k)$ in the second block. The notation $SO(n)$ (resp. $S(O(k)\times O(n-k))$) denotes the subgroup of $O(n)$ (resp. $O(k)\times O(n-k)$) with determinant $1$. Thus a point on $Gr(k,n)$ can be identified with an equivalence classes of $n$-by-$n$ special orthogonal matrices $[Y] \subset SO(n)$ where two elements $Y,Y' \in SO(n)$ are in the same equivalence class, written $Y \sim Y'$, if there exists an $M$ such that $Y' = YM$ where
\begin{equation}
M = \left[\begin{array}{cc}
        M_k&0\\
        0&M_{n-k} 
        \end{array}\right ]  
\end{equation}
such that $M_k \in O(k)$, $M_{n-k} \in O(n-k)$ and these matrices satisfy $\det (M_k) \cdot \det (M_{n-k})=1$. If $Y \sim Y'$ then $[Y]=[Y']$ denote the same point on the Grassmann manifold $Gr(k,n)$.
One advantage of this characterization of $Gr(k,n)$ is that we can utilize the well-studied geometry of $SO(n)$. It is well known that a geodesic path on $SO(n)$, starting at a point $Q\in SO(n)$, is given by a one parameter exponential flow $:t \mapsto Q\exp(tH)$ where $H$ is an $n$-by-$n$ skew-symmetric matrix. Since $Gr(k,n)$ is a quotient manifold of $SO(n)$ by the subgroup $S(O(k)\times O(n-k))$, it can be readily verified that when representing geodesics on $Gr(k,n)$, one can further restrict $H$ to be a skew symmetric matrix of the form
\begin{equation}
H = \left [\begin{array}{cc} 0_k&-B^T\\B & 0_{n-k} \end{array} \right ], B \in \mathbb{R}^{(n-k) \times k}
\end{equation}
where the size and location of the zero-blocks mirror the size and location of $M_k,M_{n-k}$ in the block diagonal matrix $M$. A geodesic on $Gr(k,n)$, starting at the point $[Q]\in Gr(k,n)$, can thus be expressed in parameterized form as: 
\begin{equation}
    Q(t) = Q \exp(t \left [ \begin{array}{cc} 0&-B^T\\B & 0  \end{array}\right ]).
\end{equation}
The sub-matrix $B$ specifies the direction and the speed of the geodesic path. More details can be found in~\cite{edelman1998geometry}. As will be seen later in Section~\ref{sec:flag_matrix_representation}, an advantage of the characterization of the Grassmannian as an equivalence class of special orthogonal matrices is that this approach allows a straightforward generalization for defining and representing points and geodesics on a {\it flag manifold} thanks to the underlying Lie theory. 

Computations of distances between points on the Grassmannian $Gr(k,n)$ are often performed using an $n$-by-$k$ orthonormal matrix representative (whose column space corresponds to the point on $Gr(k,n)$). In this setting, a point on $Gr(k,n)$ can be represented as an equivalence class of $n$-by-$k$  orthonormal matrices where $X\sim X'$ iff $X' = XU$ where $U\in O(k)$. The distance between two points on $Gr(k,n)$ (i.e. two $k$-dimensional subspaces of $\mathbb{R}^n$) $[X]$ and $[Y]$ can be computed via the compact SVD of $X^TY$, i.e.,
\begin{equation}\label{eq:grassmann_distance}
U\Sigma V^T \defeq X^TY.
\end{equation}
From the SVD, the geodesic distance between $[X]$ and $[Y]$ is defined as:
\begin{equation}\label{eq:gr_distance}
d_g([X],[Y]) = \sqrt{\sum_{j=1}^k \lambda_j^2}
\end{equation}
where $\lambda_j = \arccos(\sigma_j)$ with $\sigma_j$ denoting the $j^{th}$ diagonal element of $\Sigma$. In the formula $(XU)^TYV = \Sigma$, the columns of $XU$ and $YV$ are the principal vectors between $[X]$ and $[Y]$. The geodesic between $[X]$ and $[Y]$ rotates the columns of $XU$ to the columns of $YV$ while the diagonal elements of $\Sigma$ encode the cosine of the angles between these corresponding columns. 

%---------------------------------------------------------------------------------------------------%
\bigskip

%\todo{KIRBY add comment for figure 3 in text.}
\begin{figure}[!htbp]
\centering
\includegraphics[width=12cm]{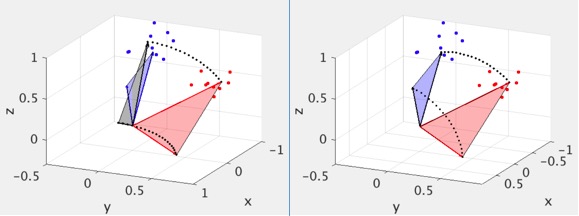}
\caption{A comparison of geodesics on the Grassmannian (left) and flag (right) manifolds for representing the distance between two data sets.
The red subspace is being moved to the blue subspace via the gray subspace in each case.}
\label{GrVflag}
\end{figure}

\section{The Flag Manifold}
\label{FLtext}
The distinction between geodesics
on Grassmannians and flags is
captured pictorially
in Figure \ref{GrVflag}.  For Grassmannians, one is moving a subspace into another subspace
along the shortest trajectory.
In the flag setting, this trajectory 
has to remain faithful to the nesting
structure of the subspaces.  In 
Figure \ref{GrVflag} (right) we see 
the required flag alignment of the coordinate directions in the 2D subspace whereas no alignment
is required for the Grassmannian (left).  The details and ramifications of this difference are elucidated below.

\subsection{Flags and their appearance in data analysis}
A $flag$ of subspaces in $\mathbb R^n$ is a nested sequence of subspaces $\{\mathbf{0}\} \subset \mathbf{V_1}\subset \mathbf{V_2}\subset \cdots \subset \mathbf{V_d} = \mathbb{R}^n$. The signature or type of the flag is the sequence $(\dim{\mathbf V_1}, \dim{\mathbf V_2}, \dots, \dim{\mathbf V_d})$. This dimension information can be also be encoded as the sequence $(\dim{\mathbf V_1}, \dim{\mathbf V_2}-\dim{\mathbf V_1}, \dim{\mathbf V_3}-\dim{\mathbf V_2}, \dots, \dim{\mathbf V_d}-\dim{\mathbf V_{d-1}})$. In this paper, we will use this second type of encoding for the signature of  a flag, thus we will identify the type  of  a flag in $\mathbb R^n$ by the sequence of positive integers $(n_1, n_2, \dots, n_d)$ where $\mathrm{dim}\, V_j = \Sigma_{i=1}^j n_i$ and $n_1+n_2+\cdots+n_d = n$. We let $FL(n_1,n_2,\dots,n_d)$ denote the {\it flag manifold} whose points parameterize all flags of type $(n_1,n_2,\dots,n_d)$. As a special case, a flag of type $(k,n-k)$ is simply a $k-$dimensional subspace of $\mathbb{R}^n$ (which can be considered as a point on the Grassmann manifold $Gr(k,n)$). Hence $FL(k,n-k)=Gr(k,n)$. The idea that the flag manifold is a generalization of the Grassmann manifold will be utilized in Section~\ref{sec:flag_matrix_representation} to introduce the geodesic formula on the flag manifold (see \cite{2019arXiv190700949Y} for a nice expanded development of the geodesic formula). The nested structure inherent in a flag appears naturally in the context of data analysis. 
\begin{enumerate}
     \item Multi-resolution analysis: the wavelet decomposition of data  into components in a nested sequence of vector spaces
     also has a flag structure.
     Each {\it scaling} subspace $V_j$ is a dilation of its adjacent neighbor $V_{j+1}$ in the sense that if
     $f(x) \in V_j$ then there is a reduced resolution copy $f(x/2) \in V_{j+1}$ \cite{mallat_89b,mallat_89c,daubechies}. In brief, the  sequence of nested scaling subspaces $\cdots \subset V_2 \subset V_1 \subset V_0 \subset V_{-1} \subset \cdots$
     can be viewed as a point on a flag manifold. 
     
      \item SVD basis of a real data matrix:
        Let $X \in \mathbb{R}^{n\times p}$ be a real data matrix consisting of $p$ samples living in $\mathbb R^n$. The left singular vectors $U$ obtained from the compact SVD, $X = U\Sigma V^T$, determine an ordered basis for the column span of $X$. The order is based on the magnitude of singular values. This order provides a straightforward way to associate a flag to $U$. For example, to associate a flag with signature $(1,1, \dots, 1)$ to $U = [u_1|u_2|\dots |u_k]$, we construct the nested sequence of subspaces $\mathrm{span}([u_1])\subsetneq \mathrm{span}([u_1|u_2])\subsetneq \cdots \subsetneq \mathrm{span}([u_1|\cdots|u_k])\subsetneq \mathbb{R}^n$. This flag of type $ (1,1,\dots,1,n-k)$ in $\mathbb{R}^n$ corresponds to a point $[U]$ on $FL(1,1, \dots ,1,n-k)$. As will be discussed in Section~\ref{results}, using an SVD 
        basis of a data set to produce a flag with a given signature can provide additional information when comparing data sets.  
\end{enumerate}
%---------------------------------------------------------------------------------------------------%
\subsection{Representation of the flag manifold}\label{sec:flag_matrix_representation}
The flag manifold $FL(n_1,n_2,\dots,n_d)$ parametrizes all flags of type $(n_1,n_2,\dots,n_d)$. The presentation in \cite{edelman1998geometry} describes how to view the Grassmann manifold $Gr(k,n)$ as the quotient manifold $O(n)/O(k)\times O(n-k)$. 
Similarly, we can view a flag manifold as a quotient manifold constructed from $O(n)$. In particular, $FL(n_1,n_2,\cdots,n_d) \cong O(n)/O(n_1)\times O(n_2) \times \cdots \times O(n_d)$ where $ n_1 + n_2 + \cdots + n_d = n$. In this definition, $O(n_1)\times O(n_2) \times \cdots \times O(n_d)$ denotes the subgroup of $O(n)$ consisting of block diagonal matrices with elements from $O(n_k)$ in the $k^{th}$ block. Although it is common to represent a flag manifold as a quotient manifold of $O(n)$, it is more convenient to represent a flag manifold as a quotient manifold of $SO(n)$ for the purposes of computations involving the $\exp$ map (since $\exp(H)\in SO(n)$ for any skew-symmetric matrix $H$). Hence for the computations in this paper, we make the representation $FL(n_1,n_2,\cdots,n_d) \cong SO(n)/S(O(n_1) \times \cdots \times O(n_d))$. Let $Q\in SO(n)$ be an $n$-by-$n$ orthogonal matrix, the equivalence class $[Q]$, representing a point on the flag manifold, is the set of orthogonal matrices
\[
[Q]=
\left \{
Q
\left[
\begin{array}{cccc}
M_1 & 0 &\cdots & 0\\
0 & M_2 &\cdots &0\\
\vdots & &\ddots &\vdots \\
0 & \cdots &&M_d
\end{array}
\right]
\right \}
\]
where $\sum_{i=1}^d n_i = n$, $M_i \in O(n_i)$ and $\prod_{i=1}^{d} \det(M_i)=1$.

\subsubsection{Example: $FL(1,1,1)$}
As a special case, a flag of type $(1,1,\cdots,1)$ is called a full flag and $FL(1,1,\cdots,1)$ is the full flag manifold in $\mathbb{R}^n$. In Figure~\ref{fig:full_flag}, we present a visualization of the nested structure of a full flag in $\mathbb{R}^3$, namely a $1$-dimensional line living in a $2$-dimensional plane living in $\mathbb{R}^3$. The set of all such flags is $FL(1,1,1) \cong O(3)/O(1) \times O(1) \times O(1)$. From the perspective of comparing data sets, Figure~\ref{fig:ellisoidal_data} shows that the SVD basis of ellipsoidal data points corresponds to a flag on $FL(1,1,1)$. Let $[u_1,u_2,u_3]\in O(3)$ be the SVD basis of  some ellipsoid ordered by the corresponding singular values, here $u_1$,$u_2$,$u_3$ are simply the major, median and minor axis respectively and $[u_1,u_2,u_3]$ is a flag representation of the ellipsoid data set. Comparing two ellipsoids amounts to measuring the geodesic distance between the two corresponding flags on $FL(1,1,1)$.     

\begin{figure}[!htbp]
\centering
\includegraphics[width=8cm]{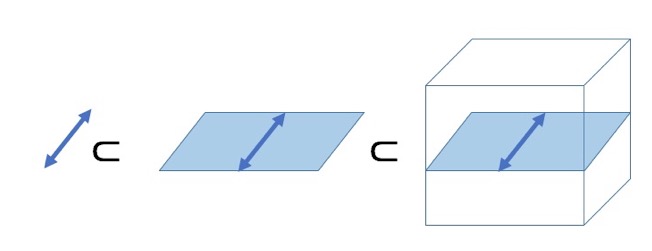}
\caption{A visualization of a full flag in $\mathbb R^3$.}
\label{fig:full_flag}
\end{figure}

\subsection{Tangent space at $[Q]$ to $FL(n_1,n_2,\cdots,n_d)$}
Let $Q$ be an element of $SO(n)$ and let $(n_1, n_2, \dots, n_d)$ be any sequence of positive integers which add up to $n$. We can use $Q$ to build a flag with signature $(n_1, n_2, \dots, n_d)$. In doing this, we can consider $Q$ as a representative for a point $[Q]$ in $FL(n_1,n_2,\cdots,n_d)$. A tangent vector at $Q \in SO(n)$ can be decomposed uniquely as a component in a direction that does not modify the nested sequence of subspaces and a component in an orthogonal direction that does. The latter represent a tangent vector to $FL(n_1,n_2,\cdots,n_d)$ at $[Q]$.  It can be readily computed that tangent vectors in directions that preserve the flag $[Q]$ correspond to $n$-by-$n$ block diagonal skew-symmetric matrices of  the form:
\begin{equation}\label{eq:flag_vertical}
G = \left[
\begin{array}{cccc}
G_1 & 0 &\cdots & 0\\
0 & G_2 &\cdots &0\\
\vdots & &\ddots &\vdots \\
0 & \cdots &&G_d
\end{array}
\right],
\end{equation}
where $G_i$ is an $n_i$-by-$n_i$ skew-symmetric matrix. The  span of matrices of this form is sometimes called the vertical space of the quotient manifold. The horizontal space is defined to be the orthogonal complement to the vertical space with respect to the standard inner product on matrices.
Thus, the horizontal space consists of matrices of the form:
\begin{equation}\label{eq:flag_tangent}
H=
\left[
\begin{array}{cccc}
\mathbf{0}_{n_1} &-B_{2,1}^T &   \cdots    &-B_{d,1}^T      \\
B_{2,1}       & \mathbf{0}_{n_2} &       & -B_{d,2}^T    \\
     \vdots            &                  &\ddots &  \vdots   \\
   B_{d,1}             &    B_{d,2}            &  \cdots &\mathbf{0}_{n_d}
\end{array}
\right]
\end{equation}
where $ \mathbf{0}_{n_i}$ denotes an $n_i \times n_i$ matrix of zeros and $B_{i,j}\in \mathbb{R}^{n_i \times n_j}$. Elements in the horizontal space correspond to elements in the tangent space to $FL(n_1,n_2,\cdots,n_d)$ at $[Q]$, i.e. to  elements in $T_{[Q]}FL(n_1,n_2,\cdots,n_d)$.

%---------------------------------------------------------------------------------------------------%
\subsection{Geodesic and distance: $\exp$ and $\log$ map}
We now describe the exponential map and logarithmic map in the setting of flag manifolds.
\subsubsection{Exponential map}
As is mentioned earlier, a geodesic path on $SO(n)$ starting at a point $Q$ is given by an exponential flow $Q(t) = Q \exp(t  X)$ where $ X \in {\mathbb R}^{n \times n}$ is any skew-symmetric matrix. Viewing $FL(n_1,n_2,\dots,n_d)$ as a quotient manifold of $SO(n)$, one can show that a geodesic on $SO(n)$ is also a geodesic on $FL(n_1,n_2,\dots,n_d)$ as long as the skew symmetric matrix $X$ points in a direction that is perpendicular to the orbit determined by $S(O(n_1)\times O(n_2) \times \cdots \times O(n_d))$. This leads one to conclude that a geodesic path on $FL(n_1,n_2,\dots,n_d)$ at $[Q]$ is an exponential flow of the form $Q(t) = Q \exp(t H)$
where $H$ takes the form in~\eqref{eq:flag_tangent}.

Since each flag is an equivalence class of matrices, $Q(t)$ is just one of the possible representations of a given geodesic flow. Each geodesic flow emanating from $[Q]\in FL(n_1,n_2,\cdots,n_d)$ has the form
\begin{equation}
[Q(t)] = \left \{ Q\exp(t H) \left[
\begin{array}{cccc}
M_1 & 0 &\cdots & 0\\
0 & M_2 &\cdots &0\\
\vdots & &\ddots &\vdots \\
0 & \cdots &&M_d
\end{array}
\right]
\right \}
\end{equation}
where $M_i \in O(n_i)$ and $\prod_{i=1}^{d} \det(M_i) = 1$. Equipped with the metric induced by the inner product $<A,B>=\frac{1}{2}Tr(A^TB)$, we can compute the length of the path between $[Q(0)]$ and $[Q(1)]$ along the geodesic determined  by $H$
\begin{align}
Length_H([Q(0)],[Q(1)])% &= \int_{t = 0}^1 \sqrt{\dfrac{1}{2}\mathrm{Tr}(\dot{Q}(t)^T\dot{Q}(t))}\\
%                &= \int_{t = 0}^1 \sqrt{\dfrac{1}{2}\mathrm{Tr}((H^TQ(t)^TQ(t)H)}\\
                &= \sqrt{\dfrac{1}{2}\mathrm{Tr}(H^TH)} 
                 = \sqrt{\dfrac{1}{2}\sum_{j=1}^l \lambda_j^2} \label{eq:flag_distance}
\end{align}
where $\{\pm i\lambda_j\}$ are the eigenvalues of $H$. This mapping of a tangent vector (based at $[Q]$) to the flag manifold is referred to as the exponential map which in this paper is found by applying the matrix exponential.

\subsubsection{Logarithmic map}
In data analysis, it is often the case that one is given data sets or representations of data sets (e.g. through an SVD basis) and one wants to measure their similarity. If the representation of  the data is given as an orthonormal matrix, $M$, one can consider $M$ as a representative  for a point  $[M]$ on a flag  manifold. An interesting feature of flag manifolds is that  there are  typically many geodesics between points. In order to measure the distance between two points on a flag  manifold, one needs to find the length of the {\it shortest} geodesic between their representations. In order to do this, one needs to find a tangent vector, $H$,  that achieves the smallest value for $<H,H>$  among all tangent vectors determining a geodesic between the points. This tangent vector is found via the inverse operation of the exponential map (referred  to as the logarithmic map). In this section we will present an iterative algorithm which approximates the tangent vector for given representatives and, by iterating through different representatives, leads to a method to measure the distance between two flags. 
Let $[Q_0],[Q_1]$ be two points on $FL(n_1,n_2,\cdots,n_d)$. Determining a tangent vector which can be used to construct a geodesic from $[Q_0]$ to $[Q_1]$ boils down to solving the following equation
\begin{equation}
Q_1  = Q_0 \exp(H) \left[
\begin{array}{cccc}
M_1 & 0 &\cdots & 0\\
0 & M_2 &\cdots &0\\
\vdots & &\ddots &\vdots \\
0 & \cdots &&M_d
\end{array}
\right]
\end{equation}
where $M_i \in O(n_i)$ and $\prod_{i=1}^{d} \det(M_i) = 1$ and $H$ takes the form in~\eqref{eq:flag_tangent}. One can simplify this equation by multiplying on the left with $Q_0^T$. We obtain $Q=\exp(H)Q^{\prime}$ where $Q=Q_0^TQ_1$ and $Q^\prime$ denotes the  block diagonal  matrix above.

Instead of solving for $H$ directly, we modify our objective so that we are solving
\begin{equation}\label{eq:iterative_equation}
Q = \exp(H) \exp(G),
\end{equation}
where $G$ has the form in~\eqref{eq:flag_vertical} and $H$ has the form in~\eqref{eq:flag_tangent}.  We propose an iterative alternating algorithm to solve~\eqref{eq:iterative_equation}. First we introduce two projections $\mathrm{P}_H(\cdot)$ and $\mathrm{P}_G(\cdot)$, which project any $n$-by-$n$ skew-symmetric matrix to be of the forms in~\eqref{eq:flag_tangent} and~\eqref{eq:flag_vertical} respectively. The idea of the algorithm is to fix $H$ then solve for $G$ alternating with fix $G$ then solve for $H$. Given an initial guess for $G$, call it $G^{(0)}$, we can solve for $H$, i.e. $\hat{H} = \log(Q\exp(-G^{(0)}))$ and then project $\hat{H}$ to its desired form to obtain $H^{(1)} = \mathrm{P}_H(\hat{H})$. Similarly, we approximate $G$ as $G^{(1)} = \mathrm{P}_G(\exp(-H^{(1)})Q)$ and iterate. Here we present the iterative alternating algorithm in Algorithm~\ref{alg:non_oriented_iterative_exp_log_algorithm}. It is important to note that in these computations, we work implicitly on the fully oriented flag manifold $SO(n)/SO(n_1)\times SO(n_2) \times \cdots \times SO(n_d)$. There is a natural $2^{d-1}$ to $1$ map from the fully oriented flag manifold to the flag manifold. For each of these $2^{d-1}$ elements on the fully oriented flag manifold, that descend to the same element  on   the flag manifold, we apply the iterative alternating algorithm. All that is left to  do  is  to  pick the "optimal" $H$, i.e. the one with the smallest value of $<H,H>$, among the $H$ arising as output from the iterative alternating algorithm. Each converged solution of the iterative alternating algorithm corresponds to a geodesic on the fully oriented flag manifold. Since $<H,H>$ measures the length of the geodesic determined by $H$, we are picking the shortest length among these geodesics. It is worth noting that in carrying  out this algorithm, we are also solving the distance problem on {\it any} partially oriented flag  manifold (but that  is a story for another day). Algorithm~\ref{alg:full_oriented_flag} is presented to sample all representations of a given flag on the fully oriented flag manifold. Thus one cycles through representatives generated by Algorithm~\ref{alg:full_oriented_flag}, feed  these into Algorithm~\ref{alg:main_algorithm}, and pick the $H$ which has the smallest value for $<H,H>$. 

An overview of the main algorithm is presented as follows,
\begin{itemize}
\item[1.] Present two (special) orthogonal matrix representations (of data sets) $X_1, X_2 \in SO(n)$ and the flag structure $\mathbf{p} = \{n_1,\cdots,n_d\}$ to the algorithm. Move $X_1$ to the origin (identity): $Q = X_2^TX_1$.
\item[2.] Compute all $2^{d-1}$ elements of $Q$ in the fully oriented manifold via Algorithm~\ref{alg:full_oriented_flag}: $\{Q_i\}_{i=1}^{2^{d-1}}$ = generateQi(Q,\textbf{p})
\item[3.] For each element $Q_i \in \{Q_i\}_{i=1}^{2^{d-1}}$, solve Equation~\eqref{eq:iterative_equation} using Algorithm~\ref{alg:non_oriented_iterative_exp_log_algorithm}: $H_{i}^{(j)},G_{i}^{(j)}$ = iterativeSolver($Q_i$,\textbf{p}), iterate this process $M$ times, i.e. $j=1,\cdots,M$. Find the solution associated with the minimum distance: $H_i^{*} = \arg \min \sqrt{\frac{1}{2} \mathrm{Tr}(H_i^{(j)T}H_i^{(j)}}) $ to obtain the \emph{shortest} geodesic (on the corresponding partially oriented flag).
\item [4.] Among all the \emph{shortest} geodesics on partially oriented flags, find the \emph{shortest} geodesic on the fully oriented flag: $H^* = \arg \min \sqrt{\frac{1}{2} \mathrm{Tr}(H_i^{*T}H_i^{*}})$
\end{itemize}
The pseudo code for the main algorithm is presented in Algorithm~\ref{alg:main_algorithm} calling subroutine Algorithm~\ref{alg:non_oriented_iterative_exp_log_algorithm} and Algorithm~\ref{alg:full_oriented_flag}.

\RestyleAlgo{boxruled}
\LinesNumbered
\begin{algorithm}
	\SetKwInput{KwData}{Input Data}
	\KwData{$X_1,X_2 \in SO(n)$,  $\mathbf{p}=(n_1,n_2,\dots,n_d)$,M,maxIter,$\mathrm{\epsilon}$}
	\SetKwInput{KwData}{Output Data}
	\KwData{$H^*$, $G^* $}
	\SetKwInput{KwData}{Define}
	\KwData{d(H) = $\mathrm{\sqrt{\dfrac{1}{2}\mathrm{Tr}(H^TH)}}$}
	 \SetKwFunction{FMain}{main}
	 \SetKwProg{Fn}{Function}{:}{}
	 \Fn{\FMain{$X_1$, $X_2$, $\mathbf{p}$}}{
	$\mathrm{Q} = \mathrm{X_1}^T\mathrm{X_2}$\\
	$\mathrm{d^*} = \mathrm{\infty}$\\
	$\mathrm{\{Q_i\}_{i=1}^{2^{d-1}}}$ = generateQi(Q,$\mathbf{p}$) \\
	\For{$\mathrm{Q}$ in $\mathrm{\{Q_i\}_{i=1}^{2^{(d-1)}}}$}{
	\For{$\mathrm{i = 1, \cdots,M}$}{
	H, G = iterativeSolver(Q,\textbf{p},maxIter,$\mathrm{\epsilon}$)\\
	
	\uIf{$\mathrm{d}^*$ > $\mathrm{d(H)}$}{
	$\mathrm{d}^*$, $\mathrm{H}^*$, $\mathrm{G}^*$ = d(H), H, G}
	}
	}
	\KwRet{$\mathrm{d^*},\mathrm{H^*}$,$\mathrm{G^*}$}
}
\caption{Main algorithm}
\label{alg:main_algorithm}
\end{algorithm}

\RestyleAlgo{boxruled}
\LinesNumbered
\begin{algorithm}
	\SetKwInput{KwData}{Input Data}
	\KwData{$\mathrm{Q \in SO(n)}$,  $\mathbf{p}=(n_1,n_2,\dots,n_d)$, maxIter, $\mathrm{\epsilon}$}
	\SetKwInput{KwData}{Output Data}
	\KwData{$\mathrm{H^{(k)}}$, $\mathrm{G^{(k)} }$}
	 \SetKwFunction{FMain}{iterativeSolver}
	 \SetKwProg{Fn}{Function}{:}{}
	 \Fn{\FMain{$\mathrm{Q}$, $\mathbf{p}$},$\mathrm{maxIter}$,$\mathrm{\epsilon}$}{
      Generate random $\mathrm{G^{(0)}}$\\
	 k = 0\\
	\While{$\mathrm{k \leq iterMax}$ $\mathbf{and}$  $\mathrm{err < \epsilon}$}
	{
	$\mathrm{k = k + 1}$\\
	$\mathrm{H^{(k)}} = \mathrm{P_H(\log(Q\, \exp(-G^{(k-1)})))}$\\
	$\mathrm{G^{(k)}} = \mathrm{P_G(\log(\exp(-H^{(k)})Q))}$\\
	$\mathrm{err = \|Q-\exp(H)\exp(G)\|_F}$\\
	}
	
	\KwRet{$\mathrm{H^{(k)}}$,$\mathrm{G^{(k)}}$}
	
}
\caption{Iterative Alternating algorithm}
\label{alg:non_oriented_iterative_exp_log_algorithm}
\end{algorithm} 

\RestyleAlgo{boxruled}
\LinesNumbered
\begin{algorithm}
\SetKwFunction{FMain}{generateQi}
	  \SetKwProg{Fn}{Function}{:}{}
	  \Fn{\FMain{$\mathrm{Q}$,$\mathbf{p}$}}{
	colHeader = [0,cumsum($\mathbf{p}$)]+1\\
	m = length(colHeader)\\
	n = floor(d/2)\\
	i = 1\\
	$\mathrm{Q_i}$ = $\mathrm{Q}$\\
	\For{$\mathrm{j = 1 : n}$}{
	C = nchoosek(colHeader, 2*j)\\
	\For{k = 1: size(C,1)}{
		$\mathrm{i = i+1}$\\
		$\mathrm{Q_i}$ = $\mathrm{Q}$\\
		$\mathrm{Q_i}$(:, C(k,:)) = -$\mathrm{Q_i}$(:, C(k,:))\\
	}	
	}	
\KwRet{$\mathrm{\{Q_i\}_{i=1}^{2^{(d-1)}}}$}
}
\caption{Fully-oriented flag representations(MATLAB pseudo code)}
\label{alg:full_oriented_flag}
\end{algorithm}

\begin{figure}[!htbp]
\centering
\includegraphics[height = 6cm, width=8cm]{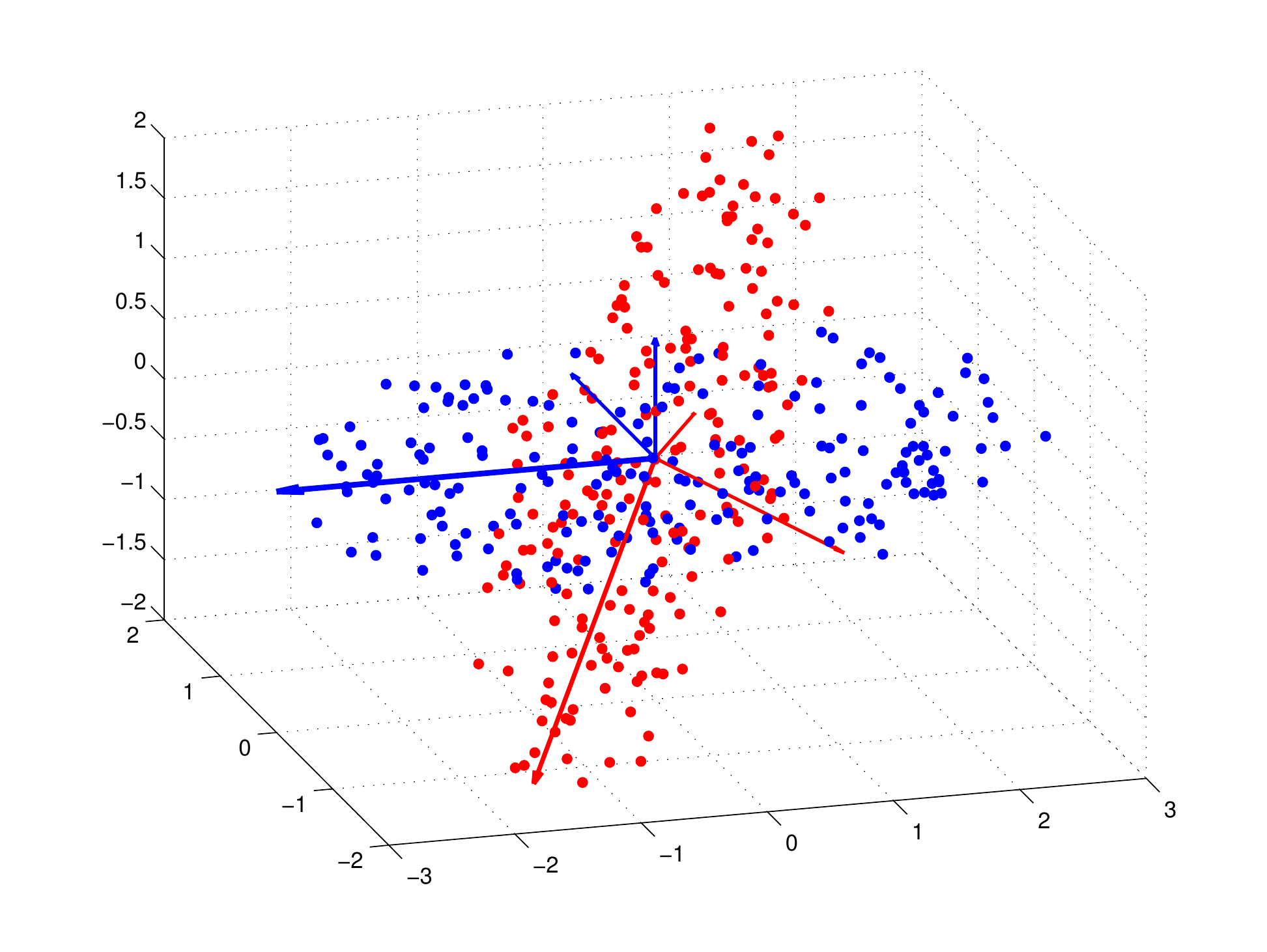}
\caption{Two sets of ellipsoid shaped data points in $\mathbb{R}^3$. Each SVD basis can be viewed as a point on $FL(1,1,1)$}.
\label{fig:ellisoidal_data}
\end{figure}

\subsection{2k Embedding}
For many practical applications, the trailing $n_d$ columns are not of interest, e.g. computations on $FL(k,n-k)=Gr(k,n)$ are usually performed using $n$-by-$k$ orthonormal matrices since only the first $k$ columns are of interest. Here in this section we will prove that the iterative algorithm~\ref{alg:non_oriented_iterative_exp_log_algorithm} can be performed in a lower dimensional space if $k = \sum_{i=1}^{d-1} n_i$ is relatively small, more specifically, if $k<n/2$. 

Without loss of generality, the geodesic between two flags of type $(n_1,n_2,\cdots,n_d)$ can always be identified with a geodesic between the identity matrix, $I$, and some $Q \in SO(n)$ by moving the initial point to $I$, i.e.,
\begin{equation}
    Q = I \exp(\left [\begin{array}{cc}A&-B^T\\B&0\end{array}\right ])
\end{equation}
where $k=\sum_{i=1}^{d-1} n_i$, $B\in \mathbb{R}^{(n-k)\times k}$ and $A$ is a $k$-by-$k$ skew-symmetric matrix of the form
\begin{equation}
    A = \left[
\begin{array}{cccc}
\mathbf{0}_{n_1} &  -B_{2,1}^T                &\cdots       &-B_{d-1,1}^T      \\
B_{2,1}                 & \mathbf{0}_{n_2} &       & -B_{d-1,2}^T      \\
\vdots                 &                  &\ddots &   \vdots    \\
   B_{d-1,1}             &     B_{d-1,2}             &  \cdots     &\mathbf{0}_{n_{d-1}}
\end{array}
\right].
\end{equation}
$Q(t) = I \exp(t \left [\begin{array}{cc}A&-B^T\\B&0\end{array}\right ]), t \in [0,1]$ traces an $n$-by-$n$ representation of the geodesic flow between $[I]$ and $[Q]$. The following theorem and its corollary provides a method to perform the iterative algorithm~\ref{alg:non_oriented_iterative_exp_log_algorithm} with $2k$-by-$2k$ matrices instead of $n$-by-$n$ matrices.
\begin{thm}\label{thm:2k_span}
Let $[Q]\in FL(n_1,n_2,\cdots,n_d)$.  Suppose $Q(t) = \exp(t \left [\begin{array}{cc}A&-B^T\\B&0\end{array}\right ])$ with $Q(0)=I$, $Q(1)=Q$ is a flag geodesic flow between $[I]$ and $[Q]$. If 
\begin{equation}\label{eq:tall_flag_geodesic}
q(t) = \exp(t \left [\begin{array}{cc}A&-B^T\\B&0\end{array}\right ])I_{n,k}
\end{equation}
and $\mathrm{span}\{q(0)\} \cap  \mathrm{span}\{q(1)\} = \{0\}$, then for all $t\in [0,1]$, $\mathrm{span}\{q(t)\} \subset \mathrm{span}\{[q(0),q(1)]\}$, where $k = \sum_{i=1}^{d-1}n_i$ and $I_{n,k}$ denotes the first $k$ columns of an $n$-by-$n$ identity matrix.
\end{thm}
Note that if $2k \geq n$, Theorem~\ref{thm:2k_span} is trivial. So here we assume $2k < n$.
Before proving the theorem, we need to introduce some notation. Let $q \defeq QI_{n,k} = q(1)$ be the first $k$ columns of $Q$. In fact, $q(t)$ defined in Equation~\eqref{eq:tall_flag_geodesic} can be understood as a geodesic path between $I_{n,k}$ and $q$ by viewing $FL(n_1,n_2,\cdots,n_d)$ as a quotient manifold of the Stiefel manifold $St(k,n)$ (refer to~\cite{2019arXiv190700949Y} for more details). Further, we write the  $n$-by-$k$ orthonormal matrix $q$ in block matrix form as
\begin{equation}\label{eq:defn_q}
q = \left [ \begin{array}{c} q_k\\ q_{n-k}\end{array}\right]
\end{equation}
where $q_k$ and $q_{n-k}$ denote the first $k$ rows and the trailing $n-k$ rows of $q$ respectively.
\begin{lemma}\label{lemma:same_span}
If $q(t)$ is defined as in Equation~\eqref{eq:tall_flag_geodesic}, such that $q(0) = I_{n,k}$ and $q(1) = q$, then $\mathrm{span}\{q_{n-k}\} = \mathrm{span}\{B\}$. 
\end{lemma}
\begin{proof}
Let $U_BR_B \defeq B$ be the compact QR decomposition of $B$ ($U_B$: $(n-k)$-by-$k$, $R_B$: $k$-by-$k$).
Define 
\begin{equation}
f(t) = (I-U_BU_B^T)Jq(t)
\end{equation}
where $J = \left [\begin{array}{cc} 0& I_{n-k} \end{array} \right]$ is the last $n-k$ rows of the $n$-by-$n$ identity matrix. Hence left multiplication by $J$ on $q(t)$ simply selects the last $n-k$ rows of $q(t)$. By definition $f(0) = 0$. Differentiate $f(t)$ to get: 
%\begin{align}
%\dot{f}(t) & = (I-U_BU_B^T)J \left[ \begin{array}{cc} A&-B^T\\B&0 \end{array}\right] q(t)\\
%           & = (I-U_BU_B^T)\left [\begin{array}{cc}B&0 \end{array}\right] q(t) \\
%           & = 0
%\end{align}
\begin{align}
\dot{f}(t) & = (I-U_BU_B^T)J \left[ \begin{array}{cc} A&-B^T\\B&0 \end{array}\right] q(t) = 0
           %& = (I-U_BU_B^T)\left [\begin{array}{cc}B&0 \end{array}\right] q(t) \\
  %         & = 0
\end{align}
Therefore, $f(t)\equiv 0$ for $t\in[0,1]$. If we evaluate $f(t)$ at $t=1$, we get:
\begin{align}
f(1) %& = (I-U_BU_B^T)J q \\
    & = (I-U_BU_B^T)q_{n-k} = 0
\end{align}
By the assumption that $q(0)$ and $q(1)$ do not intersect, we know $q_{n-k}$ is of rank $k$ hence $U_B$ is also of rank $k$. The conclusion follows.
 \end{proof}
Now we present a proof to the theorem.
\begin{proof}
Let $UR \defeq [I_{n,k},q]$ be the thin QR-decomposition of $[q(0),q(1)]$. Consequently, $U$ is an orthonormal basis for $\mathrm{span}\{[q(0),q(1)]\}$. The $n$-by-$k$ orthonormal matrix $U$ takes the block form
\begin{equation}
U = \left [ \begin{array}{cc}I_k&0\\ 0&C\end{array} \right].
\end{equation}
Note that $\mathrm{span}\{C\} = \mathrm{span}\{q_{n-k}\}$ where $q_{n-k}$ is defined in Equation~\eqref{eq:defn_q}.
Define
\begin{equation}
g(t) = (I-UU^T)q(t).
\end{equation}
By definition, $g(0) = (I-UU^T)I_{n-k} = 0$. If we differentiate $g(t)$, we get:
\begin{align}
\dot{g}(t) %&=(I-UU^T)\left[ \begin{array}{cc} A&-B^T\\B&0 \end{array}\right]q(t)\\
          %  & = \left[ \begin{array}{cc} 0&0\\0&I_{n-k}-CC^T \end{array}\right]\left[ \begin{array}{cc} A&-B^T\\B&0 \end{array}\right]q(t)\\
            &=\left[ \begin{array}{cc} 0&0\\(I_{n-k}-CC^T)B&0 \end{array}\right]q(t)
\end{align}
By Lemma~\ref{lemma:same_span}, $\mathrm{span}\{B\} = \mathrm{span}\{q_{n-k}\} = \mathrm{span}\{C\}$. We conclude that $\dot{g}(t) \equiv 0$, which implies $g(t) \equiv 0$. Therefore $q(t)$ is always living in the $\mathrm{span}$ of $[q(0),q(1)]$.
\end{proof}
The theorem shows that the flag geodesic flow $q(t)$ between $I_{n,k}$ and $q$ never leaves the $2k$-dimensional subspace $\mathrm{span}\{[I_{n,k},q]\}$, which leads to the conclusion that the logarithmic map computation can be performed within this $2k$ dimensional space without loss of information. Here we introduce the following corollary.    
 \begin{cor}
 Suppose $q(t)$ is defined as in Equation~\eqref{eq:tall_flag_geodesic} such that $q(0)=I_{n,k}$ and $q(1) = q$. Let $UR \defeq [I_{n,k},q]$ be the compact QR-decomposition of $[q(0),q(1)]$, then $\phi(t) = U^Tq(t)$ is a geodesic flow between $\phi(0) = U^Tq(0)$ and $\phi(1) = U^Tq(1)$ on $FL(n_1,n_2,\cdots,n_{d-1},k)$. Moreover, $d(\phi(0),\phi(1)) = d(q(0),q(1))$ and $q(t) = UU^T\phi(t)$. 
 \end{cor}
This corollary can be proved by combining the results from Theorem~\ref{thm:2k_span} and Corollary 2.2 in~\cite{edelman1998geometry}.
%---------------------------------------------------------------------------------------------------%

%\section{Simple Examples}

%$$S(O(1)xO(1)xO(1))/S()1))$$

\section{Numerical Experiments}
\label{results}
%\todo{We need a synthetic example which shows a case when Grassmann fails but flag works, put Indian Pine results in }

\subsection{Ellipsoid data}
The purpose of this synthetic example is to show the difference between flag geodesic and Grassmannian geodesic, as well as their corresponding geodesic distance under the context of  comparing data sets. As can be seen in Figure~\ref{fig:ellisoidal_data}, each ellipsoid data cloud contains 100 data points in $\mathbb{R}^3$. Let $\{r_i\}$ and $\{b_i\}$ denote the data points in the red and blue ellipsoid respectively. Each data set can be written as a short wide data matrix $[r_1,r_2,\cdots,r_{100}] = R \in \mathbb{R}^{3 \times 100}$ and $[b_1,b_2,\cdots,b_{100}] = B \in \mathbb{R}^{3 \times 100}$. We denote the SVD basis for each ellipsoid data set by $U_R = [u_R^{(1)},u_R^{(2)},u_R^{(3)}]$ and $U_B = [u_B^{(1)},u_B^{(2)},u_B^{(3)}]$. One can view the SVD basis as giving the major, medium, and minor axes of the corresponding ellipsoid. 

The Grassmannian geodesic distance between two bases is $0$ since the columns of $U_R$ or $U_B$ span all of $\mathbb{R}^3$. To compare two ellipsoids via the Grassmannian setting, one would typically represent the data sets with their first principal components namely $u_R^{(1)}$ and $u_B^{(1)}$, and then compute the distance between these two vectors on $Gr(1,3)$. Hence the Grassmannian geodesic between two ellipsoids is the path between two major axes and the distance is the angle between the major axes. The information contained in the relationship between the other two axes is lost. Note that this limitation comes from the Grassmannian rather than the data itself.

By representing two ellipsoids of data points by their SVD bases $U_R$, $U_B$ such that $[U_R]$, $[U_B] \in FL(1,1,1)$, one has finer resolution to describe the corresponding ellipsoids since $FL(1,1,1)$ has dimension $3$ (while $Gr(1,3)$ has dimension $2$). The geodesic between two flag representations correspondingly encodes more information than moving one major axis to another in the Grassmannian setting. 

\subsection{MNIST image data set}
Here we utilize the well-studied MNIST data set to illustrate the use of the flag manifold for comparing sets of SVD bases of "mixed" digits.
We select hand written digits "1" and "5" from the training set of the MNIST data set, where each digit is a $28 \times 28$ image. 
All images are vectorized and centered by subtracting the mean of all images.
Then we form a set of mixed digits data sets consisting of two classes, namely "major 1$/$minor 5" and "major 5$/$minor 1". 
"major 1$/$minor 5" (resp. "major 5$/$minor 1") is formed by concatenating $m$ "1"'s (resp. $m$ "5"'s) and $p$ "5"'s (resp. $p$ "1"'s). In general $m$ is assumed to be larger then $p$. 
Hence each data set is represented by a $784 \times (m+p)$ matrix. We compute the SVD basis for each $784 \times (m+p)$ matrix and select the first $k$ columns of the SVD basis as a representation for each data set. 
Thus each data set is represented by a $784 \times k$ orthonormal matrix. For the following experiment $m=16$, $p=9$ and $k=5$. 
We may consider each $784 \times 5$ SVD basis as a data point on $Fl(2,3,779)$ or $Gr(5,784)$. The first $5$ eigen-digits for both of the two classes in this experiment are demonstrated in Figure~\ref{fig:major_5_minor_1} and Figure~\ref{fig:major_1_minor_5}. One can compute the pairwise flag and Grassmannian geodesic distance to form the corresponding distance matrix. 
We then embed these data points to the Euclidean space by multi-dimensional scaling. 

In Figure~\ref{fig:comparison_flag_gr_mnist}, we see the configurations of MDS using Grassmannian(\ref{fig:gr_MNIST}) and flag distance(\ref{fig:flag_MNIST}). We observe that in \ref{fig:gr_MNIST}, the Grassmannian MDS configuration is showing overlapping between two classes. This is not surprising since each data point, no matter which class, is capturing the span of $"1"$'s and $"5"$'s. As can be seen in \ref{fig:flag_MNIST}, there is a clear separation between two classes except for one point. Note the input matrices fed to the algorithm are identical for both configurations. The difference is purely coming from the effect of the flag structure.

\begin{figure}[!htbp]
\centering
\begin{subfigure}{.45\textwidth}
\centering
\includegraphics[width=6cm,height=6cm]{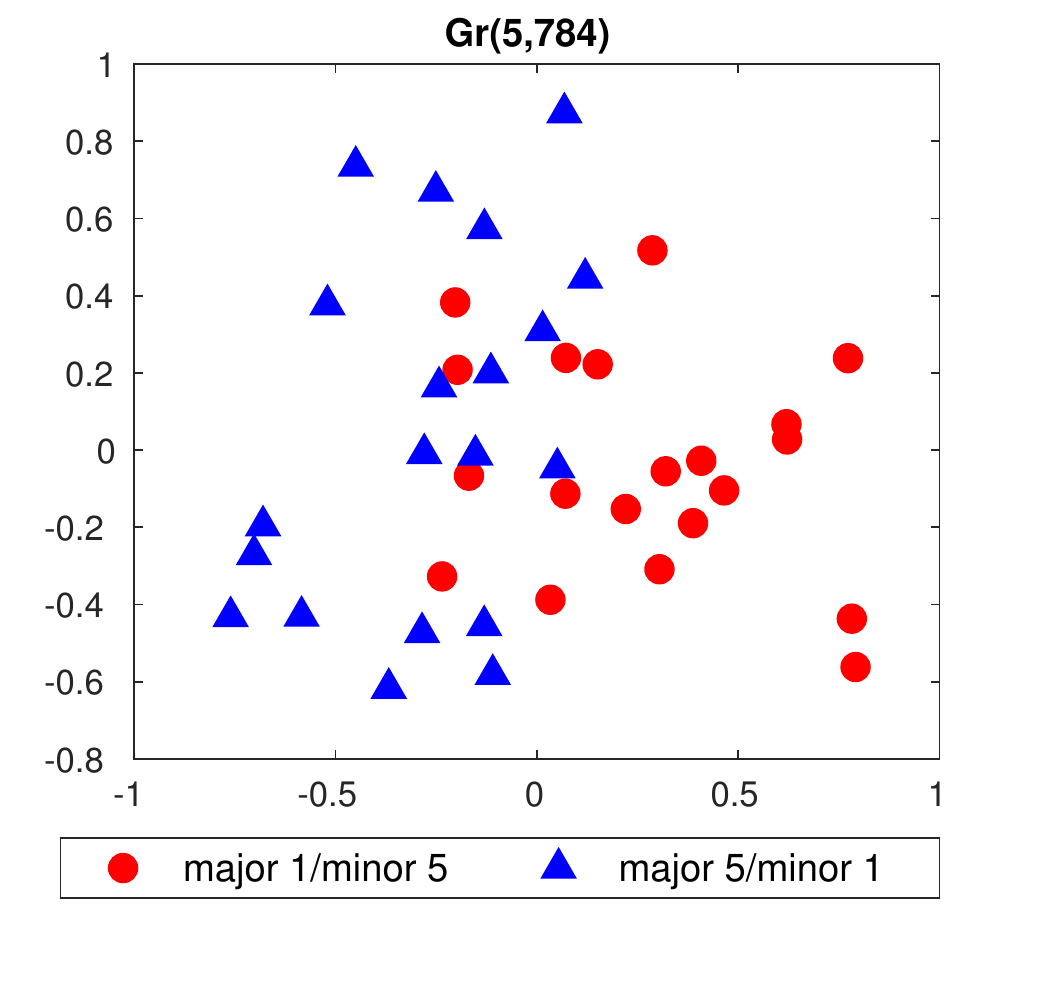}
\caption{Grassmannian MDS configuration}
\label{fig:gr_MNIST}
\end{subfigure}
\begin{subfigure}{.45\textwidth}
\centering
\includegraphics[width=6cm,height=6cm]{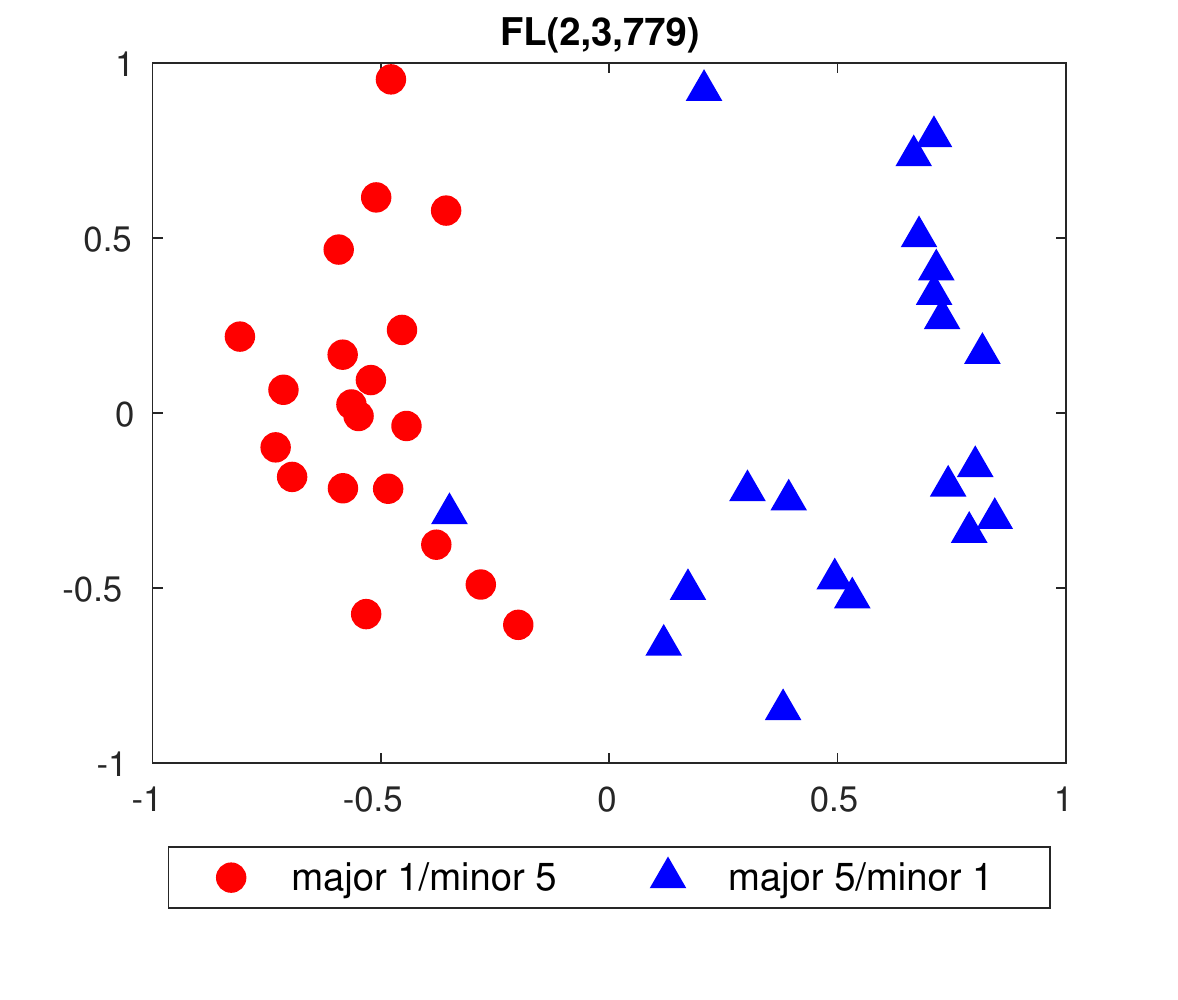}
\caption{Flag MDS configuration}
\label{fig:flag_MNIST}
\end{subfigure}
\caption{Comparison of Grassmannian and flag MDS configurations}
\label{fig:comparison_flag_gr_mnist}

\end{figure}

\begin{figure}[!htbp]
\centering
\includegraphics[width=10cm,height=2cm]{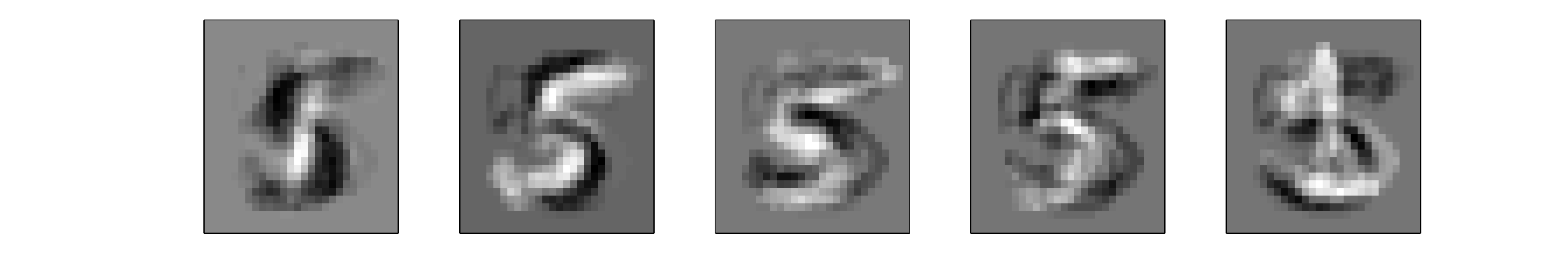}
\caption{Firt 5 eigen digits of major 5/minor 1 data set}
\label{fig:major_5_minor_1}
\end{figure}

\begin{figure}[!htbp]
\centering
\includegraphics[width=10cm,height=2cm]{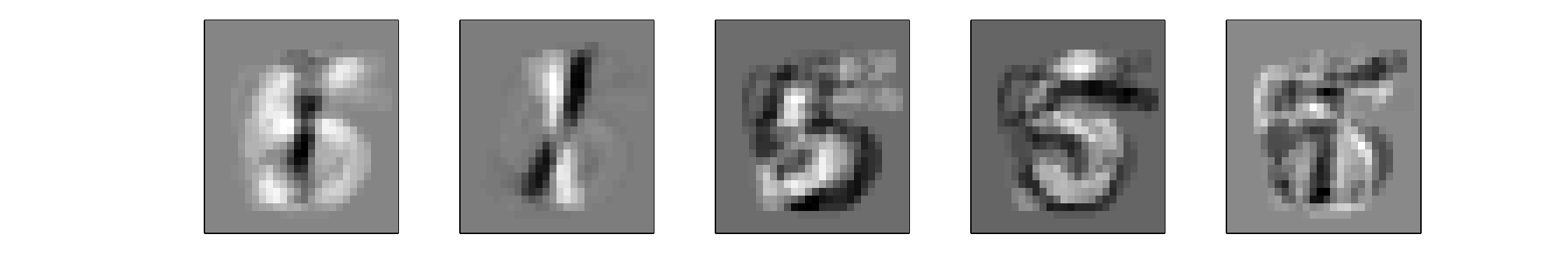}
\caption{First 5 eigen digits of major 5/minor 1 data set}
\label{fig:major_1_minor_5}
\end{figure}

\subsection{Indian Pines hyperspectral image data}
\begin{figure}[!htbp]
\centering
\includegraphics[width=8cm,height=10cm]{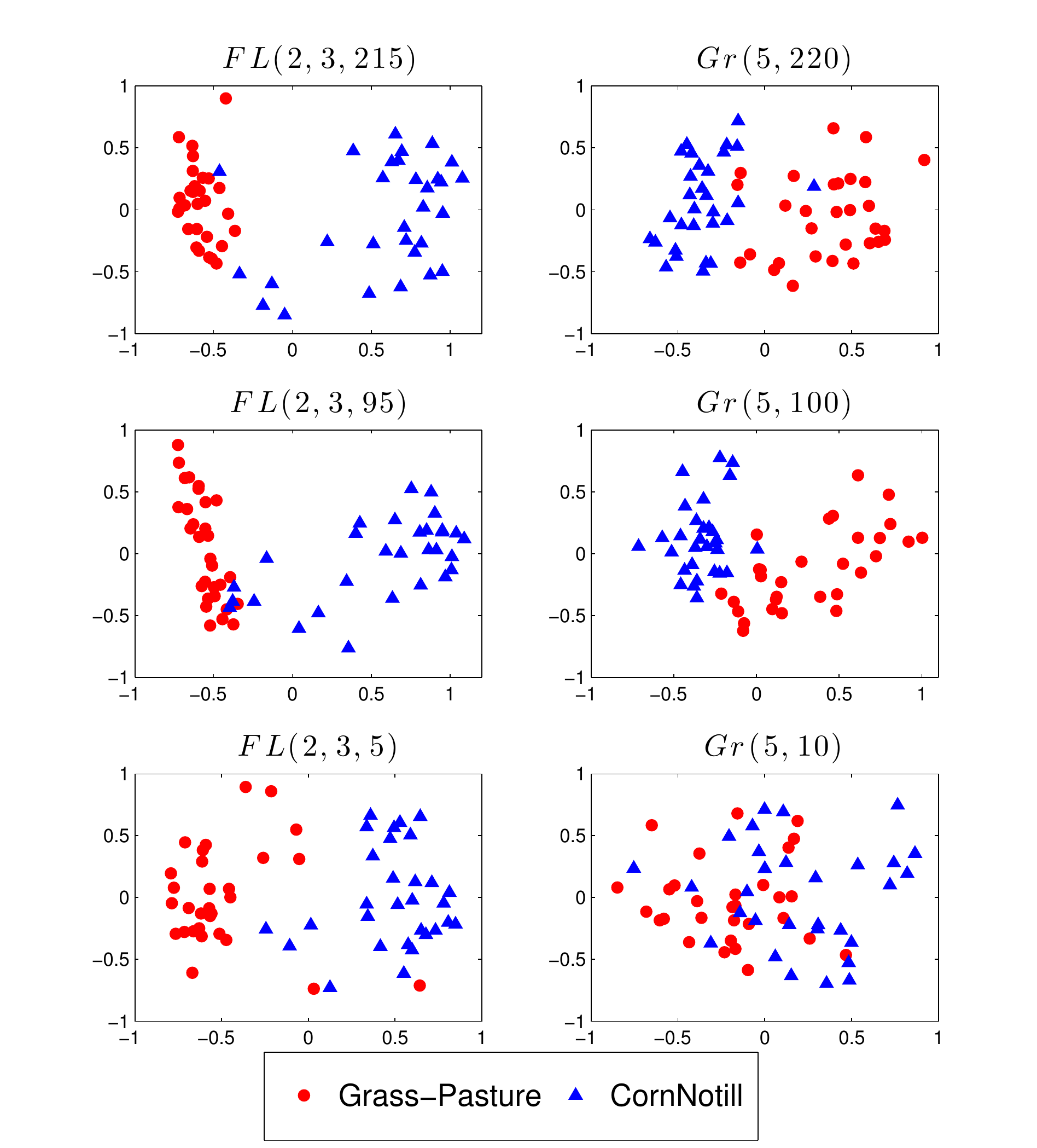}
\caption{A comparison(horizontal) of the Grassmannian and Flag manifolds for representing data sets. The subspace dimension $k$ fixed while the ambient dimension $n$ is varying from 220,100 to 10.  }
\label{fig:IP_fix_feature_dimension}
\end{figure}

To illustrate the utility of the proposed flag model in comparing real data sets, we apply it to the Indian Pines hyperspectral image data set. %\todo{maybe we want to cite some previous Grassmann work? :  WE WILL ADD THESE LATER IF THE PAPER IS 
%ACCEPTED} 
The hyperspectral images in this data set are $145 \times 145$ pixels by $220$ spectral bands (from $0.4\mu m$ to $2.4\mu m$). $10366$ pixels are labelled and each is assigned to one of the $16$ classes. Here we will test both the flag model and the Grassmann model on the task of visualizing sets of data sets. 

 For a chosen dimension $k$ (note that $k = \sum_{i=1}^{d-1}n_i$ for $FL(n_1,n_2,\cdots,n_d)$), we assemble $30$ $n \times k$ matrices $X_i$ from each class (so $p = 60$ data matrices total). Each data matrix consists of $k$ $ 200 \times 1$ data vectors which belong to one of the two classes. Then for each matrix $X_i$, a compact SVD is applied to obtain an SVD/PCA basis, hence each data point (subspace) is represented by a $220 \times k$ orthonormal matrix $U_i$ where $U_i\Sigma_iV_i^T = X_i$. The distance between SVD bases, assumed as representatives for points on a given flag manifold, can then be computed to obtain a $p \times p$ distance matrix. We use this distance matrix to embed these flags as points in Euclidean space via Multi-Dimensional Scaling (MDS). The first two coordinates of the optimal Euclidean configuration are selected for visualization in $\mathbb{R}^2$. Figure~\ref{fig:IP_fix_feature_dimension} illustrates the Euclidean embedding configurations for fixed subspace dimension $k = 5$ with various ambient dimensions using both the Grassmannian geodesic distance~\eqref{eq:gr_distance} and flag distance~\eqref{eq:flag_distance}. The ambient space is selected to be the $n$ spectral bands with highest responses for $n=100, 10, 5$. It is observed in the first two rows that both Grassmannian and flag geodesic distance provide a good separation with relatively large ambient dimension at $n = 220$ and $100$. When the ambient dimension is reduced to $n=10$, the third row of Figure~\ref{fig:IP_fix_feature_dimension} shows that the flag distance MDS embedding separates two classes in $\mathbb{R}^2$ while the Grassmannian MDS embedding shows heavy overlapping. Figure~\ref{fig:eigenvalue_plot} shows the eigenvalues corresponding to the MDS embedding using flag distance on $FL(2,3,5)$ (left) and $Gr(5,10)$ (right). As we can see, the largest eigenvalue on the left panel is dominating which also suggests that flag MDS configurations are separable in lower dimension, which we don't observe in the Grassmannian MDS eigenvalues plot. Figure~\ref{fig:various_feature_dim} shows, for fixed ambient dimension $n = 220$, how sets of data sets are pulled apart by increasing the dimension in the flag structure. From top left, we observe that the embedding of data points on $FL(1,219)$ to $\mathbb{R}^2$ live on a circle and are not separable. As we increase the flag structure dimension, the corresponding MDS configurations start to show more separation and for $FL(1,4,215)$, the embedding of two classes is linearly separable.

In Figure~\ref{fig:3class}, we select 6 bands (bands: 3,29,42,61,65,158) and use 20 pixels within the same class to form a data matrix of size $6 \times 30$. Each class consists of 20 such short and wide matrices and each matrix is represented by its $6$-by-$6$ SVD basis and assumed to  be representatives for points on $FL(2,2,2)$. The pairwise distance is computed to obtain MDS configurations on $\mathbb{R}^2$. It is observed that the MDS embeddings of 3 classes are separable in low dimensional space with only $6$ bands.

\begin{figure}[!htbp]
\centering
\includegraphics[width=9cm, height=4cm]{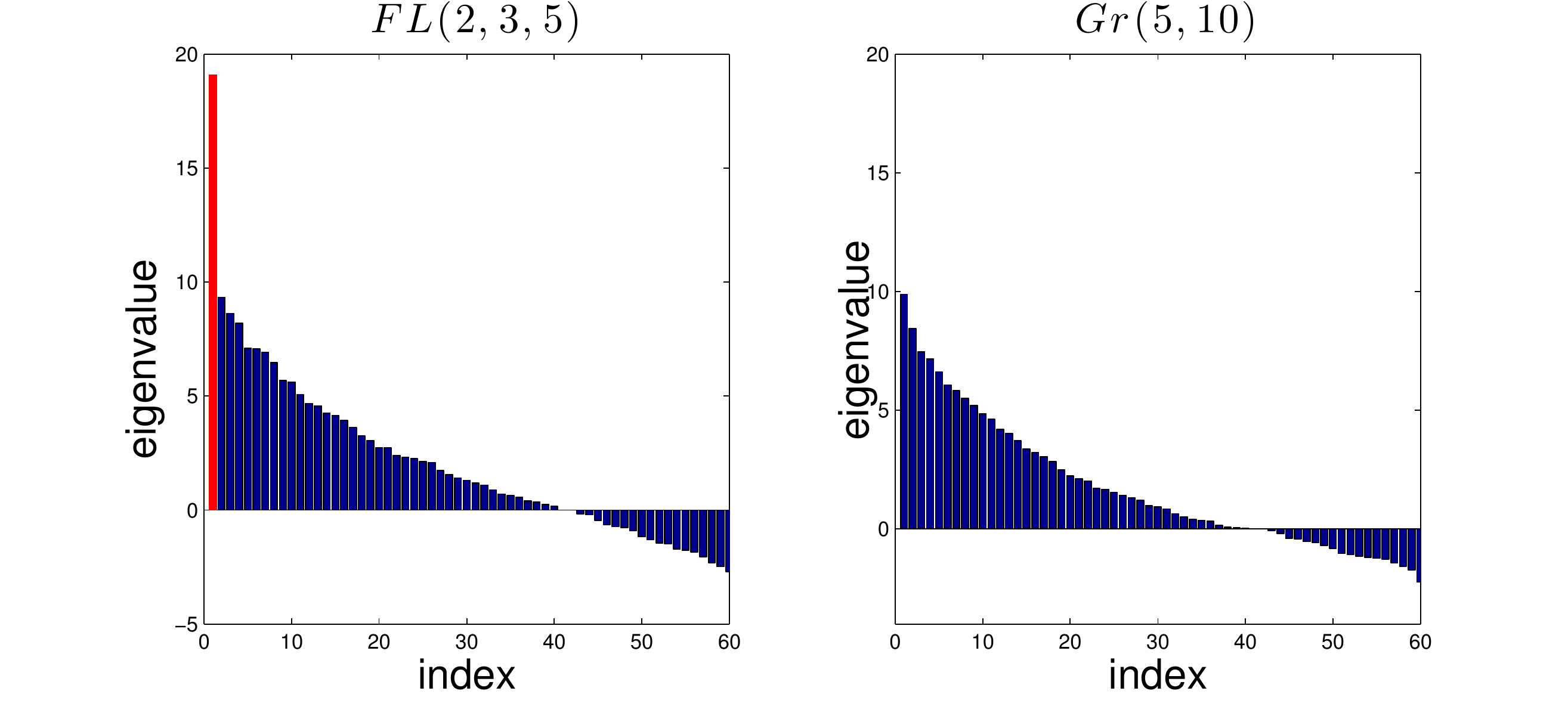}
\caption{Eigenvalues of MDS for Left:$FL(2,3,5)$, Right:$Gr(5,10)$ in descending order.}
\label{fig:eigenvalue_plot}
\end{figure}

\begin{figure}[!htbp]
\centering
\includegraphics[width=8cm,height=8cm]{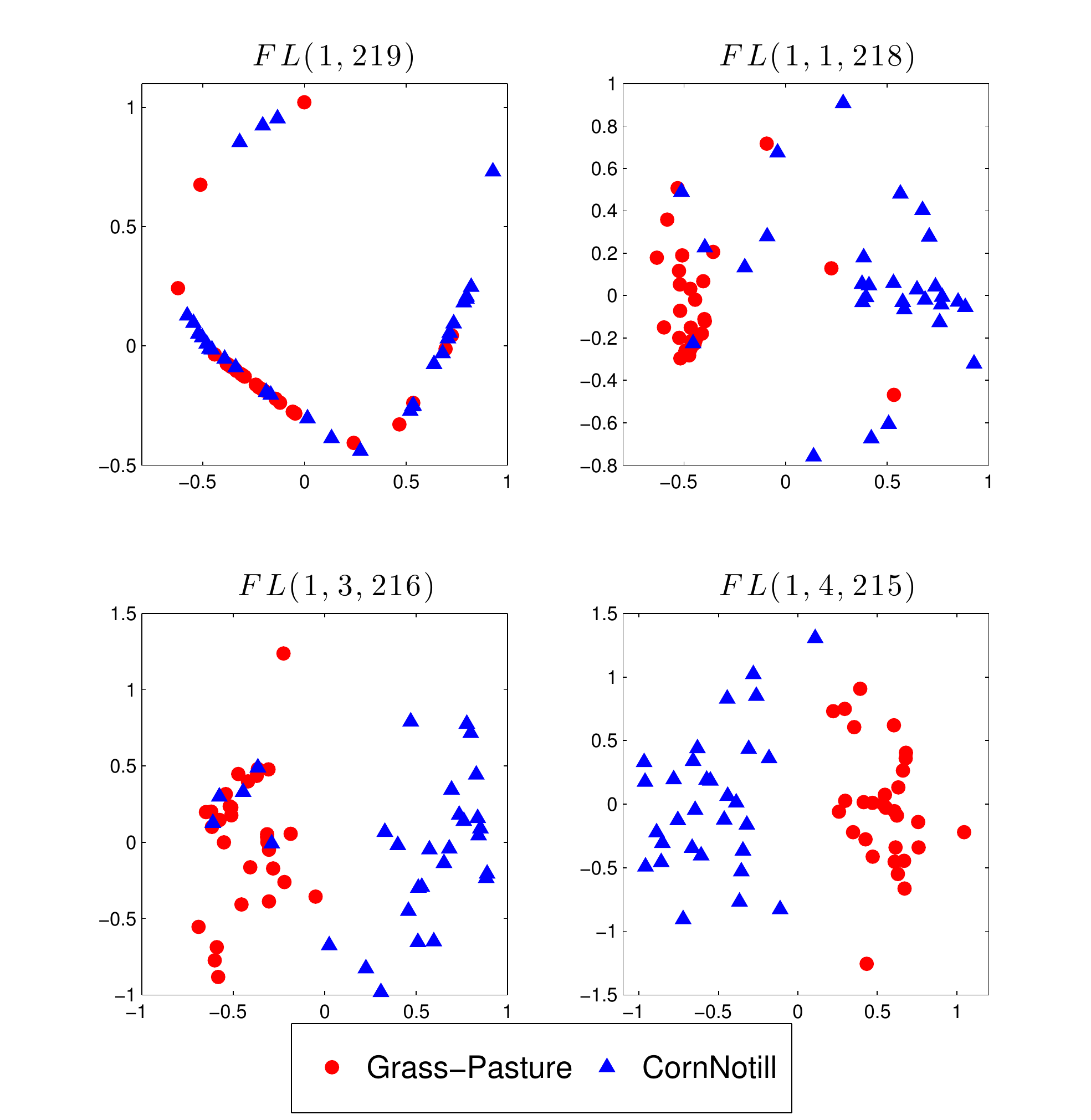}
\caption{Configuration of points on various flag manifolds embedded in Euclidean space.}
\label{fig:various_feature_dim}

\end{figure}

\begin{figure}[!htbp]
\centering
\includegraphics[width=7cm]{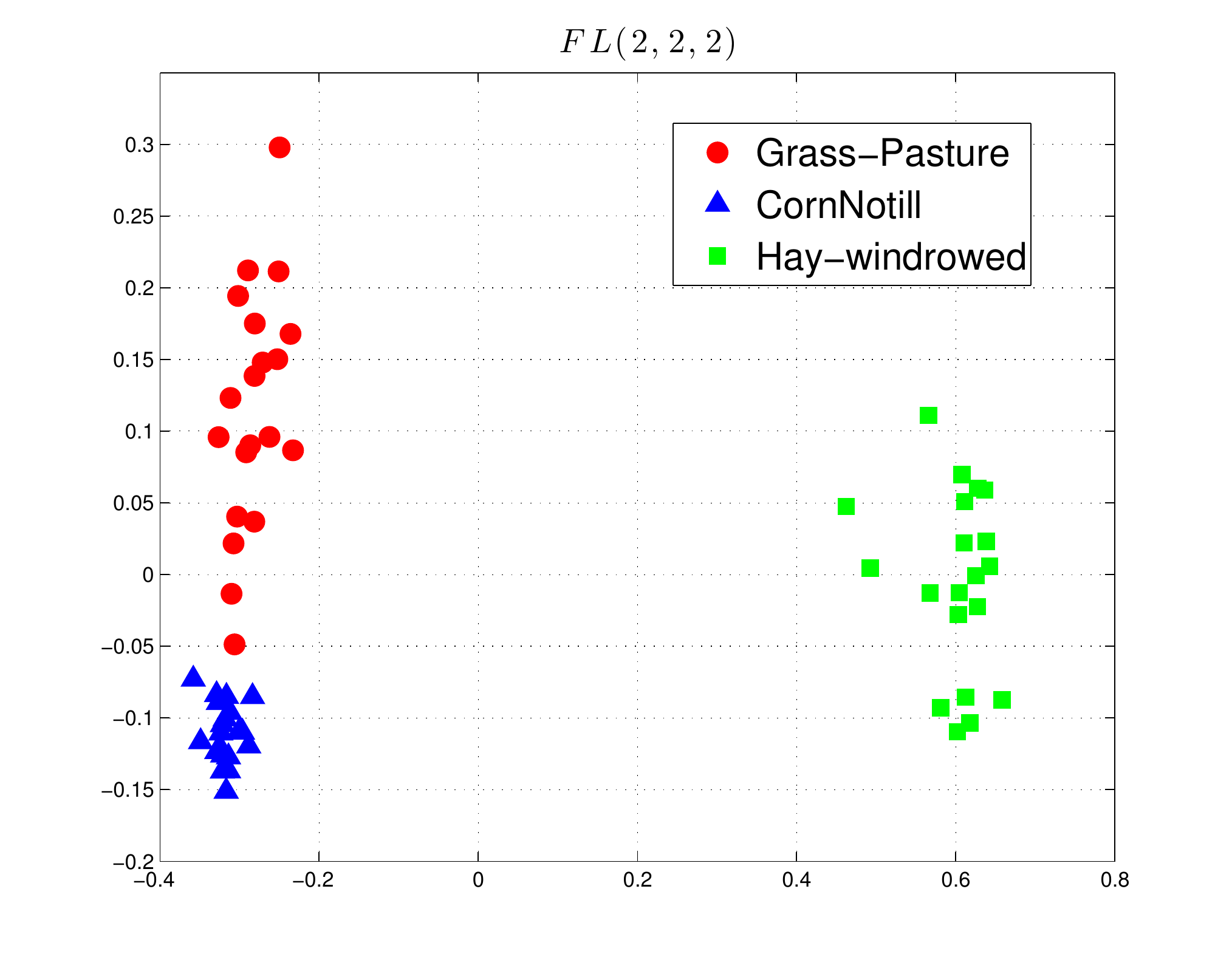}
\caption{Configuration of points on $FL(2,2,2)$ embedded in Euclidean space for 3 classes: Grass-Pasture',Corn-notill,Hay-windrowed. 6 bands(3,29,42,61,65,158) are selected so the ambient dimension $n=6$.}
\label{fig:3class}

\end{figure}

%\todo{ADD eigenvalue plot}

%---------------------------------------------------------------------------------------------------%
\section{Conclusion}
\label{conc}
We have proposed a geometric framework for comparing distances
between nested subspaces, i.e., points on a flag manifold.  This approach exploits a mathematical framework that enables the data analyst to gain insight
into the way the data resides in its ambient space, both in terms of dimension and distribution.
This approach is suitable for the analysis of wide data matrices, e.g., where the number 
of data features is less than the number of points and for data sets consisting  of a mixture of classes.

We have presented the theoretical foundation for computing
geodesic distances between two points
on a flag manifold.  The theory lends itself naturally
to numerical algorithms for computing the distance 
as well as the set of points along the shortest path between the two
points.  This formulation allows one to move a set
of nested subspaces into another set of nested 
subspaces along the  shortest path that respects
the intrinsic geometry.
%The resulting methodology provides a means
%to robustly compare data sets  with 
%fewer features than data points. 
These tools
provide a mechanism to leverage angles
between subspaces where the previous formalism
on the Grassmannian may fail.

The flag geodesic algorithms have been demonstrated on mixed MNIST data sets and on the Indian
Pines hyperspectral data set where the number
of hyperspectral features (each corresponding to a
frequency band) and 
flag structure  are varied.  In particular,
we focus on the transition from tall to wide matrices.
We see that the geodesic 
distance on the flag manifold is 
able to separate the data for visualization
in two dimensions while the Grassmannian framework
fails to do so.  %The Euclidean structure
%of each approach is revealed by the MDS eigenvalues and embedding.
%\section*{Acknowledgment}
%This paper is based on research partially supported by the National Science Foundation under Grant No. DMS-1322508 Any opinions, findings, and conclusions or recommendations expressed in this material are those of the authors and do not necessarily reflect the views of the National Science Foundation. 

\section*{Acknowledgment}
This paper is based on research partially supported by the National Science Foundation under Grants No. NSF-1633830, 
NSF-1830676, and NSF-1712788.

{\small
\bibliographystyle{plain}
\bibliography{main}
}

\end{document}